\documentclass{article}

\PassOptionsToPackage{numbers, compress}{natbib}
\usepackage[main, final]{neurips_2026}

\usepackage[utf8]{inputenc} %
\usepackage[T1]{fontenc}    %
\usepackage{url}            %
\usepackage{booktabs}       %
\usepackage{amsfonts}       %
\usepackage{nicefrac}       %
\usepackage{microtype}      %
\usepackage{xcolor}         %
\usepackage{dcolumn}%
\usepackage{bm}%
\usepackage{microtype}
\usepackage[table]{xcolor}
\usepackage{basic_commands}
\usepackage{graphicx}
\usepackage{subfigure}
\usepackage{booktabs} %
\usepackage{hyperref}
\definecolor{lightpurple}{RGB}{235,225,245}
\usepackage{mathtools}
\usepackage{amsthm}
\usepackage{amsmath}   %
\usepackage{amsfonts}  %
\usepackage{tabularx}
\usepackage{booktabs}
\usepackage{amssymb}   %
\usepackage{pifont}
\usepackage{xspace}
\usepackage{braket}
\usepackage{algorithm}
\usepackage{algorithmic}
\usepackage{wrapfig}
\usepackage{caption}
\usepackage{titlesec}

\definecolor{lightpurple}{RGB}{235,225,245}

\titlespacing\section{0pt}{10pt plus 2pt minus 1pt}{0pt plus 2pt minus 2pt}
\titlespacing\subsection{0pt}{10pt plus 2pt minus 1pt}{0pt plus 2pt minus 2pt}
\titlespacing{\paragraph}{0pt}{2pt plus 0.5pt minus 0.5pt}{0.7em}

\newcommand{\xmark}{\ding{55}}
\newcommand{\methodname}{ADA\xspace}

\newcommand{\mb}{\mu_{\text{base}}}
\newcommand{\mt}{\mu_{\theta}}
\newcommand{\Mo}{\mathcal{M}_o(\mathcal{X})}

\theoremstyle{plain}
\newtheorem{theorem}{Theorem}[section]

\theoremstyle{definition}

\theoremstyle{remark}

\usepackage[capitalize,noabbrev]{cleveref}

\title

\author{%
  Kai Tyrus Nelson$^{1}$\thanks{Denotes equal contribution} \quad Tobias Kreiman$^{1}$\footnotemark[1] \quad Sergey Levine$^{1}$ \quad Aditi S. Krishnapriyan$^{1,2}$ \\
  $^{1}$University of California, Berkeley \quad $^{2}$LBNL \\
  \texttt{\{kai\_nelson, tkreiman, sergey.levine, aditik1\}@berkeley.edu} \\
}

\begin{document}

\title{\Large Bridging the Simulation-to-Experiment Gap with Generative Models using Adversarial Distribution Alignment}
\maketitle

\author{Anonymous Authors}
\begin{abstract}

A fundamental challenge in science and engineering is the simulation-to-experiment gap. While we often possess prior knowledge of physical laws, these physical laws can be too difficult to solve exactly for complex systems. Such systems are commonly modeled using simulators, which impose computational approximations. Meanwhile, experimental measurements more faithfully represent the real-world, but experimental data typically consists of \textit{observations} that only partially reflect the system’s full underlying state. We propose a data-driven distribution alignment framework that bridges this simulation-to-experiment gap by pre-training a generative model on fully observed (but imperfect) simulation data, then aligning it with partial (but real) observations of experimental data. While our method is domain-agnostic, we ground our approach in the physical sciences by introducing Adversarial Distribution Alignment (\methodname). This method aligns a generative model of atomic positions---initially trained on a simulated Boltzmann distribution---with a distribution of experimental observations.
We prove that our method recovers the target observable distribution, even with multiple, potentially correlated observables. We also empirically validate our framework on synthetic, molecular, and experimental protein data,
demonstrating that it can align generative models with diverse observables. Our code is available at \url{https://kaityrusnelson.com/ada/}.

\end{abstract}

\maketitle

\section{Introduction}

Physical laws precisely describe the behavior of matter, but solving them is often prohibitively expensive when applied to real-world systems. A central challenge in computational science is thus the development of accurate yet tractable models. For example, numerical simulation methods in atomistic modeling trade computational cost for accuracy by introducing various approximations from the classical to the quantum level \citep{Bannwarth2019_gfn2xtb, kohn1965self, ek1966_cc, MacKerell1998_charm, Cornell1995_amber}. Despite their simplifying assumptions, by predicting physical properties \textit{in silico}, accurate computational models can reduce reliance on expensive laboratory experiments.

Meanwhile, direct experimental measurements better reflect the real world, but exact measurements of the full underlying state of a system can be costly or completely inaccessible. Consequently, experimental data often only reflects \textit{partial observations} of the real world. Although latent variable models can be used to reason about partially observed systems \citep{Wainwright2008, Bartholomew2011}, such models can be data-intensive and challenging to train in practice, especially in experimental settings where measurements are expensive. As a result, this leaves a discrepancy between abundant, fully-observed but approximate simulation data and scarce, partially-observed but more accurate experimental measurements: a discrepancy which we refer to as the \emph{simulation-to-experiment gap}.  

\begin{wrapfigure}{R}{0.5\textwidth}
\vspace{-12pt}
    \centering
    \includegraphics[width=1\linewidth]{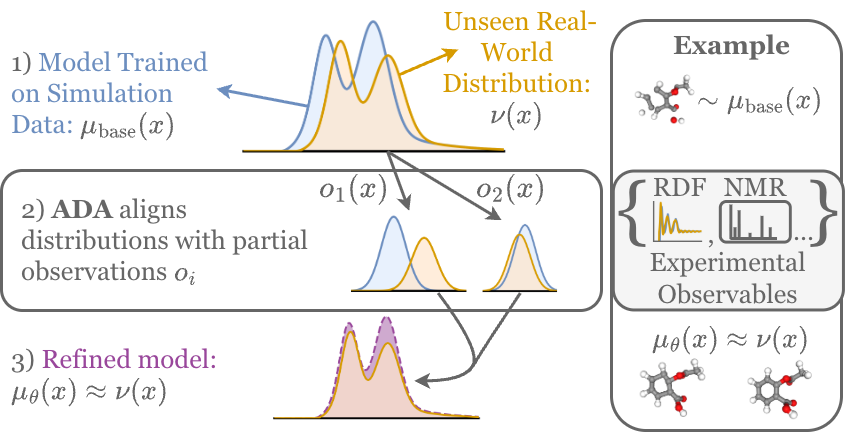}
    \caption{\textbf{\methodname aligns generative models trained on approximate simulation data with the real world by leveraging partial experimental observations.} (1) \methodname starts from a base generative model $\mb(x)$ trained on simulation data, such as molecular dynamics simulations using a classical force field. (2) \methodname aligns the base model with \textit{multiple, potentially correlated} partial experimental observations (e.g., radial distribution functions (RDFs) or nuclear magnetic resonance (NMR) measurements). (3) \methodname yields an aligned model $\mt(x)$ that approximates the true underlying real-world distribution $\nu(x)$.}
    \label{fig:teaser}
    \vspace{-12pt}
\end{wrapfigure}

To address this gap, we propose \textbf{Adversarial Distribution Alignment from Partial Observations (\methodname)}, an algorithm that first learns a generative model of fully observed simulated data, and then aligns it with partially observed experimental measurements (see \Cref{fig:teaser}). \methodname aligns the model with experimental observable \emph{distributions} via a min–max objective. Our algorithm alternates between learning to distinguish observable samples generated by the model from samples drawn from the experimental dataset, and updating the generative model accordingly. By leveraging real-world data as a scalable source of supervision, \methodname helps bridge the simulation-to-experiment gap when training machine learning models of the physical world.

Importantly, these real-world data distributions correspond to the true targets of interest in downstream applications such as materials and drug discovery, where accurate property prediction enables screening of promising candidates. For example, \methodname can align a generative model of protein structures---pre-trained on classical molecular dynamics simulations as the initial physical prior---with high-dimensional, noisy cryo-EM observations to better capture the real-world distribution of protein states (amino acid positions). We preview such an application in \Cref{sec:exp_proteins}.

Our main contribution is the introduction of \methodname as a framework that leverages physical knowledge through simulated data, while mitigating the approximation errors of purely simulation-based approaches by aligning to real-world distributions. We prove that \methodname recovers the target observable distribution, even in the presence of multiple, potentially correlated observables. We also show empirically that \methodname can align generative models to diverse molecular and protein observables, with alignment to experimental distributions improving as additional observables are incorporated.

\section{Distribution Matching with Observables}
\label{sec:method}

Our goal is to train a generative model $\mt(x)$ that matches the distribution of observables measured from an underlying system $\nu(x)$. In this context, we define observables to be any function of the state of the system $x$: 

\[
o^{(i)}: X \rightarrow O^{(i)} \subseteq \mathbb{R}^{k_i},
\]

where $k_i$ is the dimensionality of observable $i$, and $i \in \{1,\dots,m\} = I$ indexes the $m$ observables. We only assume access to the observable quantities drawn from our underlying system, and presume no direct access to the distribution of true states $\nu(x)$. We additionally note that our $m$ observables may be arbitrarily correlated.

In general, modeling $\nu(x)$ is challenging because the amount of data for the full system might be limited or impossible to access, leaving observables that only offer a \textit{partial} view of the underlying system. Therefore, to make it more feasible to model $\nu(x)$, we leverage a model, $\mb(x)$, pre-trained on \emph{fully observed} data from an imperfect simulator of the experimental distribution, e.g., data generated from a numerical simulation. Our goal is to shift the base distribution to produce a learned distribution $\mt(x)$, which satisfies the constraint:

\begin{equation}
\label{eqn:obs_constraint}
o^{(i)}_\# \mt = o^{(i)}_\# \nu, \quad \forall i,
\end{equation}
where $o^{(i)}_\# \nu$ denotes the \emph{pushforward} of the distribution $\nu$ through the observable $o^{(i)}$, that is, the distribution over observable $i$ obtained by sampling $x \sim \nu(x)$ and mapping it to $o^{(i)}(x)$. Succinctly, we wish to learn a distribution $\mt$ whose observables match the experimental observable distribution.

\begin{wrapfigure}{R}{0.5\textwidth}
    \centering
    \includegraphics[width=0.9\linewidth] {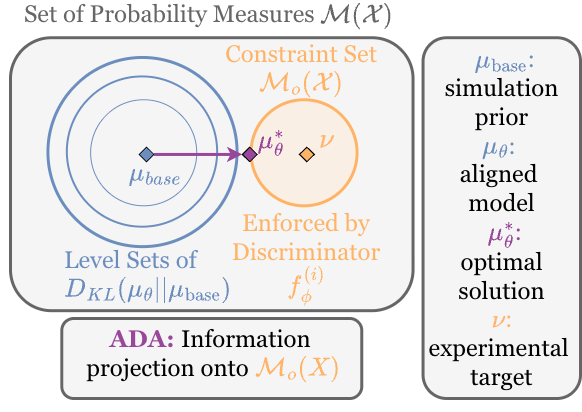}
    \vspace{-6pt}
    \caption{\textbf{
    \methodname finds the closest distribution $\mt^*$ to a base distribution $\mb$ that satisfies the observable constraints $o^{(i)}_\# \mt = o^{(i)}_\# \nu$ for all observables $i \in \{1,\dots,m\}$.} We denote the resulting feasible set by $\Mo$ (orange). The KL term in \Cref{eqn:full_objective} acts as an entropic regularizer, reflecting the fact that the constraints in \Cref{eqn:obs_constraint} do not uniquely specify the underlying distribution, since many states can map to the same observable value. The discriminator $f^{(i)}_\phi$ learns the observable Wasserstein distance to align the model with the distributional observable constraint.
    }
    \label{fig:info_proj}
    \vspace{-18pt}
\end{wrapfigure}

The objective in \cref{eqn:obs_constraint} is generally under-constrained, since the observable is only a partial observation of the underlying system, and many states could map to the same observable value. Therefore, we regularize the solution by enforcing closeness to the base distribution via a KL divergence. This acts as an informed prior, supplementing sparse and lossy real world data with reasonably accurate simulator data. If we define $\Mo = \{ \mu : o^{(i)}_\# \mu = o^{(i)}_\# \nu ; \forall i \in I\}$ to be the set of all probability measures satisfying \cref{eqn:obs_constraint},
our full objective becomes:

\begin{equation}
\label{eqn:full_objective}
\begin{aligned}
\arg \min_{\mt} \quad & D_{\mathrm{KL}}(\mt \,\|\, \mb) \\
\text{s.t.} \quad & \mt \in \Mo.
\end{aligned}
\end{equation}

From an information-theoretic perspective, \Cref{eqn:full_objective} corresponds to an information projection of the base distribution onto the set of distributions satisfying the observable constraints, with the KL term acting as an entropic regularizer (see \Cref{fig:info_proj}). We note the connection between this objective and inverse reinforcement learning (IRL), where the goal is to learn a reward function explaining expert behavior \citep{Jaynes1957, max_ent_irl}. In this framework, the learned reward implicitly distinguishes expert behavior from that induced by the current policy, and the policy is iteratively improved by optimizing this reward. In our setting, we have experimental observables instead of expert trajectories, while the policy corresponds to a generative model.
In contrast to prior work that matches only \textit{expectations} of reference features, our formulation seeks to match the \textit{full distribution} of multiple, potentially correlated observables.

\subsection{Comparison of our Problem Setting to Conditional Generation and Expectation Alignment}
\label{sec:comp_of_method}

To build intuition about our problem setting, we first contrast it with other related formulations. We discuss conditional generative modeling and expectation alignment (EA) approaches \cite{smith2026calibratinggenerativemodelsdistributional, Cesari2018, Soper1996}, which align distributions by matching observable moments, to motivate our algorithm in \Cref{sec:adapo}.

\paragraph{Comparison with conditional generative modeling.} We first consider conditional generative modeling, which explicitly factorizes the target distribution into an observable prior and a generative model conditioned on an observable value. If one had access to a paired dataset $\mathcal{D} = {(x_j, o^{(i)}_j = o^{(i)}(x_j))}$, one could attempt to satisfy \cref{eqn:full_objective} by training a conditional model $\mb(x \mid o^{(i)}_j)$ and sampling from: $\mt(x) = \int \mb(x \mid o^{(i)})d\nu(o^{(i)}).$

While this approach is viable when paired samples are available for a \emph{single} observable, it fails in the multi-observable setting when we have access only to marginal distributions
$\{\nu(o^{(i)})\}_{i=1}^m$. Having access only to marginals prevents conditioning on values sampled from the joint distribution of observables, making it impossible to simultaneously satisfy the observable and KL constraints in \Cref{eqn:full_objective}.

\paragraph{Comparison to expectation alignment methods.} We additionally contrast our problem setting, which matches full distributions, with methods that align only to observable expectations. Throughout this paper, we refer to this class of methods as \emph{expectation alignment (EA)}. In this formulation, one seeks a distribution $\mt$ that minimizes a KL divergence to a base distribution while matching the expected values of a prescribed set of observables:

\begin{equation}
\begin{aligned}
\label{eqn:EA_objective}
\arg \min_{\mt}\quad & D_{\mathrm{KL}}(\mt \,\|\, \mb) \quad
\text{s.t.}\quad  & \mathbb{E}_{x \sim \mt}[o^{(i)}(x)] = \mathbb{E}_{x \sim \nu}[o^{(i)}(x)], \quad \forall i \in I .
\end{aligned}
\end{equation}

In theory, under mild regularity assumptions, EA can recover the full distribution-matching objective. This follows by augmenting the observable set with all polynomial moments, since enforcing \Cref{eqn:EA_objective} over this family corresponds to matching all mixed moments of the observable distributions. In the limit of infinitely many moments, this uniquely identifies the observable marginals and recovers the solution to \Cref{eqn:full_objective}.

In practice, however, this approach is infeasible: the number of moments grows combinatorially with the observable dimension and degree, and exact recovery of the target distribution generally requires matching infinitely many moments (see \Cref{exp:synthetic} for an empirical example). From this perspective, matching the full observable distribution amounts to enforcing equality over a sufficiently rich class of test functions of the observables. We instead propose parameterizing a discriminator as a neural network that directly learns functions of the observables, enabling full distributional alignment beyond matching a finite number of moments.

\subsection{Adversarial Distribution Alignment with Partial Observations (ADA)}
\label{sec:adapo}

We now present our algorithm, Adversarial Distribution Alignment with Partial Observations (\methodname), which performs full distributional alignment using only partial observations. To realize \cref{eqn:full_objective}, we re-frame the objective as a dual optimization problem, learning the Wasserstein distance between generated and real experimental observables in the same way as \citet{arjovsky2017wassersteingan}. Together with the known connections between GANs and inverse reinforcement learning \citep{finn2016connectiongenerativeadversarialnetworks, ho2016generativeadversarialimitationlearning}, this motivates an adversarial optimization procedure for \Cref{eqn:full_objective}. 

Specifically, to enforce the observable constraint in \Cref{eqn:obs_constraint}, we use a Wasserstein distance over the space of observables to learn a pseudo-metric over state distributions. We will refer to this pseudo-metric as the observable Wasserstein distance. For each observable, we define the observable Wasserstein distance between the learned and target distributions as follows:

\begin{equation}
\begin{aligned}
  d^{(i)}(\mu, \nu)
    &= \sup_{f^{(i)} \in \mathrm{Lips}(1)} \Big( \mathbb{E}_{o^{(i)}_{\#} \mu}[f^{(i)}] - \mathbb{E}_{o^{(i)}_{\#} \nu}[f^{(i)}] \Big),
\end{aligned}
\end{equation}

where $\mathrm{Lips}(1)$ is the class of all Lipschitz-1 functions. In practice, this Wasserstein distance is approximated by learning a discriminator $f_\theta^{(i)}(o^{(i)}(x))$ with a local Lipschitz constant enforced by a gradient penalty \cite{arjovsky2017wassersteingan}.

We then parameterize the soft version of our constraint objective (\Cref{eqn:full_objective}) as:

\begin{equation}
\label{eqn:soft}
\begin{aligned}
\arg \max_{\mt} \quad &  - D_{\mathrm{KL}}(\mt \,\|\, \mb) - \beta \sum_{i \in I} d^{(i)}(\mt, \nu),
\end{aligned}
\end{equation}

which corresponds to the mini-max objective:
\begin{equation}
\begin{aligned}
\label{eqn:minimax}
    &\max_{\mt} \; \min_{(f^{(i)}) \in \mathrm{Lip}(1)^{|I|}} \mathcal{L}\bigl(\mt, (f^{(i)}), \beta \bigr)
    \qquad \text{where} \\
    &\mathcal{L}\bigl(\mt, (f^{(i)}), \beta \bigr)
    = -D_{\mathrm{KL}}\bigl(\mt \,\|\, \mu_{base}\bigr)  + \beta \sum_{i \in I} \left(
          \mathbb{E}_{(o^{(i)})_{\#} \mt}[f^{(i)}]
          - \mathbb{E}_{(o^{(i)})_{\#} \nu}[f^{(i)}]
      \right).
\end{aligned}
\end{equation}

We propose to alternate between learning our observable Wasserstein distance and improving our generative model to find the distribution closest to our base model that matches the target observable distribution (\Cref{fig:info_proj}).

\begin{wrapfigure}{R}{0.6\textwidth} %

\begin{small}
\vspace{-12pt}
\begin{minipage}{\linewidth}
\begin{algorithm}[H]
\caption{Observable Alignment from a Pretrained Generative Model}
\label{alg:observable_alignment}
\begin{algorithmic}[1]
\STATE \textbf{Input:}
\STATE \quad Pretrained base generator $\mu_{\text{base}}(x)$
\STATE \quad Experimental observable datasets 
$\mathcal{D}^{(i)} = \{\, o^{(i)}_j \sim (o^{(i)})_{\#}\nu \,\}$ for each observable $i \in I$
\STATE \textbf{Initialize:}
\STATE \quad Generator parameters $\theta$ with $\mu_\theta \leftarrow \mu_{\text{base}}$
\STATE \quad Observable critics $f^{(i)}_{\phi^{(i)}} : O_i \rightarrow \mathbb{R}$ for all $i \in I$
\STATE
\WHILE{not converged}
    \STATE Sample model configurations $\{x_j\}_{j=1}^{N} \sim \mu_\theta$
    
    \FOR{each observable $i \in I$}
        \STATE Compute model observables $u^{(i)}_j \leftarrow o^{(i)}(x_j)$
        \STATE Sample reference observables $\{\tilde{o}^{(i)}_j\}_{j=1}^{N} \sim \mathcal{D}^{(i)}$
    \ENDFOR
    
    \STATE Construct empirical push forward measures:
    \[
    (o^{(i)})_{\#}\mu_\theta \approx \frac{1}{N}\sum_{j=1}^{N}\delta_{u^{(i)}_j},
    \quad
    (o^{(i)})_{\#}\nu \approx \frac{1}{N}\sum_{j=1}^{N}\delta_{\tilde{o}^{(i)}_j}
    \]
    
    \[
    \begin{aligned}
    \mathcal{L}(\theta, &\{\phi^{(i)}\}) =
    -D_{\mathrm{KL}}\!\bigl(\mu_\theta \,\|\, \mu_{\text{base}}\bigr) \\
    &+\beta \sum_{i \in I}
    \Big(
    \mathbb{E}_{(o^{(i)})_{\#}\mu_\theta}[f^{(i)}_{\phi^{(i)}}]
    -
    \mathbb{E}_{(o^{(i)})_{\#}\nu}[f^{(i)}_{\phi^{(i)}}]
    \Big).
    \end{aligned}
    \]
    \STATE Update generator (ascent, use Adjoint Matching to compute $\nabla_\theta \mathcal{L}$):
    \[
    \theta \leftarrow \theta + \eta_\theta \nabla_\theta \mathcal{L}
    \]
    
    \FOR{each observable $i \in I$}
        \STATE Update critic (descent):
        \[
        \phi^{(i)} \leftarrow \phi^{(i)} - \eta_\phi \nabla_{\phi^{(i)}} \mathcal{L}
        \]
    \ENDFOR
\ENDWHILE

\end{algorithmic}

\end{algorithm}

\end{minipage}
\vspace{-18pt}
\end{small}
\end{wrapfigure}

\paragraph{Algorithm summary.} 

We summarize our overall algorithm. We first initialize our generative model as $\mt=\mb$, where $\mb$ is trained on the simulated data distribution. Starting from this initialization, our algorithm iterates the following three steps to align $\mt$ with the experimental data. (see Alg. \ref{alg:observable_alignment}). First, we sample from the current iteration of the aligned generative model $\mt(x)$. We then evaluate observables from the generated samples and compute the Lagrangian of our objective (\cref{eqn:minimax}). We finally update the generative model and the discriminators using stochastic gradient descent on the Lagrangian. We also emphasize that each discriminator may be trained independently per observable. As we will show later in \Cref{sec:theory}, this is sufficient for satisfying the desired observable constraints, even in the case of multiple, correlated observables. 

\Cref{tab:method_capabilities} summarizes the capability of \methodname and compares it with conditional generative modeling and EA approaches discussed in \Cref{sec:comp_of_method}. In short, while conditional generation can handle single observables and EA methods can align with a finite set of moments, \methodname can recover the full distribution of multiple, potentially correlated observables.\footnote{When we only have access to a distributional mean, \methodname can subsume EA approaches by restricting the function class to a linear discriminator. This is well studied in the inverse reinforcement learning literature (e.g. \citet{max_ent_irl})}

\paragraph{Practical implementation details.} We describe how we optimize the Lagrangian \Cref{eqn:minimax}, which balances the observable Wasserstein distance with the prior. We note that the Lagrangian does not need to be explicitly calculated as long as we can optimize $\mu$ and $(f^{(i)})$ with respect to it. Many methods which are capable of optimizing a generative model with respect to a reward function could work in practice. In this work, we parameterize $\mt$ and $\mb$ as diffusion models \citep{song2021scorebasedgenerativemodelingstochastic} and we use adjoint matching \cite{domingoenrich2025adjointmatchingfinetuningflow} to obtain unbiased gradient estimates of the Lagrangian with respect to $\mu$ without requiring backpropagation through the sampling process.

We note that adjoint matching requires observables to be differentiable. Although this may seem restrictive, many structural observables and energy functions can be computed efficiently in a differentiable manner, and entropy regularized reinforcement learning could be used for non-differentiable observables. \citep{Williams1992, schulman2017proximalpolicyoptimizationalgorithms}.

\subsection{Theoretical Convergence Analysis}
\label{sec:theory}

We now address the theoretical guarantees of our objective.
Our goal with the theoretical analysis is to show that the objective in \Cref{eqn:minimax}, (1) has a unique solution and (2) converges to the desired reference observable distribution $o^{(i)}_\#\nu$ as we increase the weight $\beta$ on the observable Wasserstein distance. We first build intuition by analyzing the scenario where there is only a single observable, and then provide theorems for the general case of multiple, potentially correlated observables. 

\paragraph{Single observable case.} For the case of a single observable, we can directly solve \Cref{eqn:full_objective} to find the optimal distribution. Following previous work \citep{max_ent_irl, Jaynes1957, levine2018reinforcementlearningcontrolprobabilistic}, we consider the Lagrangian of the hard constrained optimization problem. To delineate it from the Lagrangian of \Cref{eqn:soft}, we denote it by $\mathcal Z$:

\begin{equation}
\begin{aligned}
\mathcal{Z}(\mt, \lambda, \alpha)
= \int dx\, \mt(x)\log\!\frac{\mt(x)}{\mb(x)} + \int do\, \lambda(o)\bigl(\nu(o) - \mt(o)\bigr) + \alpha\!\int dx\, (\mt(x) - 1)
\end{aligned}
\end{equation}

where $\mt(o) = \int dx \delta(o - o(x)) \mt(x)$ using $\delta$ as the Dirac delta function. We also note that there is a Lagrange multiplier for each value the observable could take on ($\lambda(o)$). 

If we differentiate to find the optimal $\mt$, we find that: $\mt(x) = \mb(x)\,\frac{e^{\lambda(o(x))}}{Z}.$
To solve for $\lambda(o(x))$, we note that by definition of the Dirac delta function:
\begin{equation}
\begin{aligned}
\mt(o) = \int dx \,\delta(o - o(x)) \mt(x) 
= 
\mb(o) \frac{e^{\lambda(o)}}{Z}.
\end{aligned}  
\end{equation}
Considering our constraint on the observable distribution:
$\nu(o) = \mt(o) = \mb(o) e^{\lambda(o)}/Z,$
we see that:
$\lambda(o) = \log ( Z \; \nu(o)/\mb(o)),$
yielding a final distribution:
\begin{equation}
\mt^*(x) = \mb(x) \frac{\nu(o)}{\mb(o)}  
\end{equation}
that satisfies $\mt(o) = \nu(o)$. In short, for a single observable, \Cref{eqn:full_objective} has a solution that recovers the target observable distribution by tilting the base distribution towards regions where it underrepresents the observable density. Although solving for the closed-form tilted distribution is possible for a single observable--- since we can estimate $\nu(o)$ from samples---the same derivation in the multi-observable setting would require access to the joint distribution of experimental observables (see \Cref{sec:comp_of_method} for details). Since we assume access only to marginal observable distributions, we instead provide a different analysis to handle the multi-observable case in the following section.

\paragraph{General case.} Analyzing the general case with multiple observables requires more care. We present the following guarantees:

\newcommand{\solL}[1]{\mathrm{sol}\,\mathcal{L}\bigl( \beta_{#1}\bigr)}
\begin{theorem}[Existence and uniqueness of the saddle point]
\label{thm:saddle-point}
Assume $X$ is a compact Polish space and that $\mb$ has full support on $X$. Additionally, allow $o^{(i)}$ to be continuous for $i \in I$. For any $\beta \in \mathbb{R}$, \cref{eqn:minimax} admits a saddle point $(\mu^*, (f^{(i)})^*)$, where $\mu^*$ is unique. Consequently, the recovered distribution is well defined. We denote this solution by $\mu^* \in \solL{}$.
\end{theorem}

\begin{theorem}[Asymptotic convergence in Wasserstein]
\label{thm:wasserstein-convergence}
Assume $D_{\mathrm{KL}}(\nu \| \mu_{base}) < \infty$. Let $\mu^*_\beta \in \solL{}$ for $\beta \in \mathbb{R}^+$. Then, as $\beta \to \infty$, we have
$
d^{(i)}(\mu^*_\beta, \nu) \to 0
\quad \text{for all } i \in I.
$
\end{theorem}

\begin{theorem}[Asymptotic convergence to the constraint set]
\label{thm:constraint-convergence}
Let $(\mu_\beta^*)_{\beta>0}$ be the sequence of saddle-point solutions from \cref{thm:wasserstein-convergence}. Suppose that $d^{(i)}(\mu_\beta^*,\nu)\to 0$ for all $i\in I$. Then any weak limit point $\bar{\mu}$ of $(\mu_\beta^*)$ satisfies
$
(o^{(i)})_\# \bar{\mu} = (o^{(i)})_\# \nu
\quad \text{for all } i\in I,
$
that is, $\bar{\mu}\in \mathcal{M}_o$.
\end{theorem}

To summarize, our algorithm in theory converges and will be arbitrarily close to satisfying the constraint set for a sufficiently large $\beta$. The associated proofs can be found in \cref{apx:proofs}. Although the KL constraint on \cref{thm:wasserstein-convergence} may seem overly restrictive, we note that this constraint is satisfied for many models of interest. For instance, any set of particles in a closed and bounded box with positive density are sufficient. We also note that we have imposed no unnecessary restrictions on the push forward of our observables, only that they continuously map to real valued vectors. We have not required that they be independent or uncorrelated. In fact, as long as the experimental dataset can be realized by a real measure with finite KL divergence from the base measure, we are able to get arbitrarily close to matching the observable measures in Wasserstein distance.

\section{Related Work}

\paragraph{Generative models.} \methodname builds on previous generative modeling frameworks. We parameterize $\mb$ and $\mt$ as diffusion models \citep{song2021scorebasedgenerativemodelingstochastic, ho2020denoisingdiffusionprobabilisticmodels}, which learn to reverse a stochastic process that corrupts a data distribution into noise. Generative adversarial networks (GANs) \citep{gans} are also closely related to our work. GANs learn a target distribution by alternating between learning a discriminator, which distinguishes real samples from generated ones, and a generator trained to fool the discriminator. \methodname builds upon the GAN framework to learn distributions in the partially observed setting.

\vspace{-4pt}
\paragraph{Maximum entropy inverse reinforcement learning.} Adversarial training also plays a central role in inverse reinforcement learning (IRL) \citep{max_ent_irl, finn2016guidedcostlearningdeep}, where the objective is to infer a reward function that explains observed behavior and to extract a policy consistent with that reward. The connection between GANs and IRL has been noted in prior work \citep{finn2016connectiongenerativeadversarialnetworks, ho2016generativeadversarialimitationlearning}, where the learned reward function acts as a discriminator and the policy serves as the generator. The maximum entropy IRL formulation \citep{Jaynes1957, max_ent_irl}, which augments reward learning with an entropic regularizer, closely aligns with our problem setting. However, rather than inferring a policy from demonstrations, we seek to align a generative model with experimental observations while leveraging simulated data as a prior. Importantly, we only leverage partial observations without requiring demonstrations of the full underlying state.

\vspace{-4pt}
\paragraph{Learning from partial observations in atomistic generative modeling.}
In the context of computational chemistry and biology, Iterative Boltzmann Inversion (IBI) \citep{Soper1996, Reith2003, Hanke_2017, Matin2024} and related conditional modeling approaches (see \Cref{sec:comp_of_method}) have leveraged a single observable, such as cryo-EM images \citep{Tang2023, Tang2023_ew}, to improve models. Simpler parametrized model classes \citep{Xu2021} have also been used to learn from partial observations. More recent simulation-in-the-loop approaches \citep{raja2025stabilityaware, Fuchs_2025, Han2025, gangan2025force} use automatic differentiation to improve a machine learning force field using observable quantities. Most similar to our approach is concurrent work that aligns the outputs of a generative model using a linear reward \citep{smith2026calibratinggenerativemodelsdistributional}.

However, we emphasize that these prior methods that leverage \textit{multiple} observables typically frame their objective as an \textit{expectation alignment (EA) problem} \citep{Cesari2018, Soper1996, raja2025stabilityaware, Maddipatla2025_alphafoldguidance, kolloff2025minimumexcessworkguidance, raghu2025cryoboltz, zarrouk2025linearscalingcalculationexperimentalobservables, Hummer2015, Kofinger2019}. For instance, guidance methods \citep{kolloff2025minimumexcessworkguidance, Maddipatla2025_alphafoldguidance, raghu2025cryoboltz} modify the score of a diffusion model, often with a trained classifier, to produce samples that match the moments of observables. Similarly, \citet{smith2026calibratinggenerativemodelsdistributional} fix a linear reward signal that they use to fine-tune a generative model to match the expectation of an observable. While in theory these EA approaches could learn a full distribution by enforcing infinitely many moments of the target distribution, this is in practice infeasible and modeling only a finite number of moments leads to discrepancies, even in low dimensions (see \Cref{sec:comp_of_method} and \Cref{exp:synthetic}).      

Our problem setting differs in that we aim to match the full distribution of \textit{multiple} experimental observables rather than only low-order statistics. For example, in computational biology, matching a single, expected protein conformation is insufficient; modeling the full experimental distribution is necessary to study folding and unfolding events, compute free energies, and estimate transition rates. Furthermore, instead of relying on a single reward signal as in the IRL and conditional generative modeling formulations, we consider multiple, potentially correlated experimental observables that jointly constrain the target distribution. \methodname extends the EA and IRL formulations to handle full distributional constraints over multiple observables.

\section{Experiments}
\label{sec:exp}

We evaluate \methodname's ability to align a generative model, pretrained on simulated data, with experimental distributions using only partial observations. We first compare \methodname to EA approaches similar to \citet{smith2026calibratinggenerativemodelsdistributional} to highlight the difference between full distributional alignment and moment matching (\Cref{exp:synthetic}). Next, to obtain controlled evaluation metrics, we use simulated data at increasing levels of fidelity to quantify how alignment improves as additional observables are provided (\Cref{sec:small_molecule}). Finally, we assess \methodname's ability to bridge the gap between a generative model trained on classical force field simulations and experimental measurements from the Protein Data Bank using cryo-EM observables \citep{pdb} (\Cref{sec:exp_proteins}). We provide experimental details, including further discussion of the chosen observables, in \Cref{apx:exp_details}. While our experiments rely on simulating the observable forward model, their goal is to understand \methodname's behavior under controlled conditions that enable quantitative evaluation in multi-observable and high-dimensional settings.

\subsection{Synthetic Data}
\label{exp:synthetic}

To illustrate the limitations of EA methods, we consider a synthetic three-dimensional mixture-of-Gaussians benchmark (see \Cref{fig:mix_guassians}). We define the base distribution $\mb$ as an isotropic mixture with eight components located at the vertices of a cube, with uniform mixture weights. We define the target distribution $\nu$ by perturbing the variances and mixture weights for the Gaussians (see \Cref{apx:exp_details} for more details).

We choose observables $o^{(i)}$ as the three pairwise coordinate projections:
$\pi_{X_i,X_j}(x_1,x_2,x_3) = (x_i,x_j).
$
Despite their simplicity, these observables induce correlated, multimodal marginals.

For evaluation, we report two metrics: (1) the Wasserstein-1 distance between energy histograms of generated samples, where $E(x) = -\log \nu(x)$, and (2) the cluster assignment accuracy via the KL divergence
$
D_{\mathrm{KL}}(\mt(z) \,\|\, \nu(z)),
$
where
$\mt(z) = \int \mt(x) \nu(z \mid x)\,dx$
denotes the induced mixture weights of generated samples. Intuitively, these metrics quantify discrepancies in global energy structure and mass allocation across modes, respectively.

\paragraph{Results.} In \Cref{tab:alignment_metrics} and \Cref{fig:improvement_differences}, we find that \methodname successfully recovers the full target distribution using the correlated projection observables. In contrast, EA fails even when enforcing moments up to fourth order. Recovering a multimodal distribution from finitely many moments is not possible in general, and the number of moments that need to be enforced scales polynomially with the dimension of the problem, highlighting the fundamental limitations of expectation-based alignment.

\subsection{Correlated Small Molecule Observables}
\label{sec:small_molecule}

\begin{wrapfigure}{R}{0.5\textwidth}
\vspace{-12pt}
    \centering
    \includegraphics[width=\linewidth]{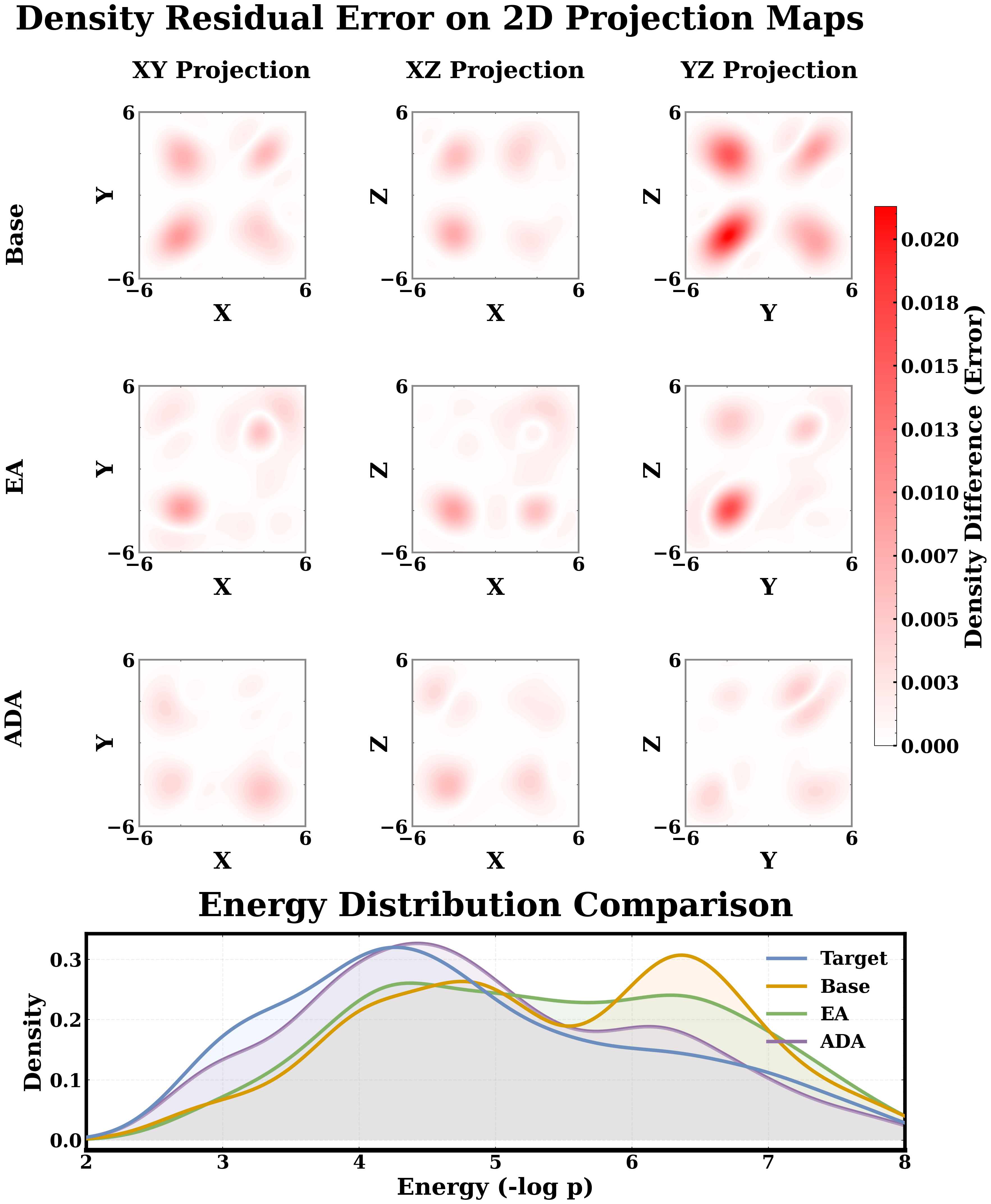}
    \caption{\textbf{Comparison between expectation alignment (EA) and \methodname on synthetic mixture-of-Gaussians benchmark.} We align a base distribution, defined as a mixture of Gaussians centered at the corners of a cube, with a target distribution with perturbed variances and mixture weights. For the observables, we use pairwise coordinate projections, yielding correlated, multimodal marginals.  We report the $\ell_1$ pdf residual between target and generated distributions before and after alignment. \methodname consistently achieves larger reductions in distributional error on multidimensional observables, even when EA uses fourth moments. This is taken from a kernel density estimate of 2000 samples.}
    \label{fig:improvement_differences}
    \vspace{-18pt}
\end{wrapfigure}

We next evaluate whether \methodname can align molecular conformational distributions using multiple correlated structural observables, and whether incorporating  observables improves alignment on both constrained and held-out quantities. As a controlled proxy for sim-to-real transfer, we consider two simulator fidelities on the MD17 small-molecule benchmark \cite{md17}, aligning a low-cost GFN2-xTB potential \cite{Bannwarth2019_gfn2xtb} to a higher-fidelity density functional theory (DFT) reference using only observable-level supervision.

We progressively add structural observables to the alignment procedure---including the mean interatomic distance (MD), radius of gyration ($R_g$), bond lengths, hydrogen-bond distance, and COOH--ester distance (details in \Cref{apx:observables})---to study \methodname’s ability to reconcile multiple correlated marginals while preserving underlying molecular structure.

For evaluation, we report two complementary metrics. First, for each observable we compute the Wasserstein-1 distance ($W_1$) between histograms of generated and target samples, measuring marginal distributional alignment.  Beyond marginal comparisons, we evaluate the       
  model's ability to recover free energy surfaces (FES) over pairs of physically meaningful observables. For each observable we bin both the reference and generated samples into a      
  joint 2D histogram, normalize each to a discrete probability      
  distribution, and compute the Jensen–Shannon divergence. We report this over three observable pairs on aspirin,
  capturing hydrogen-bonding geometry, molecular compactness, and functional group
  arrangement.

\paragraph{Results.}
\Cref{fig:experiment2_hist} and \Cref{tab:cumulative_ablation} show that \methodname consistently improves alignment to the DFT target as additional observables are incorporated. Across all observables, \methodname achieves substantially lower $W_1$ as compared to EA while simultaneously reducing FES JSD, indicating that it preserves molecular correlations while enforcing accurate marginal constraints.

\subsection{Noisy and High-Dimensional Cryo-EM Observable with Experimental Data}
\label{sec:exp_proteins}

\begin{wrapfigure}{R}{0.5\textwidth}
\vspace{-12pt}
    \centering
    \includegraphics[width=\linewidth]{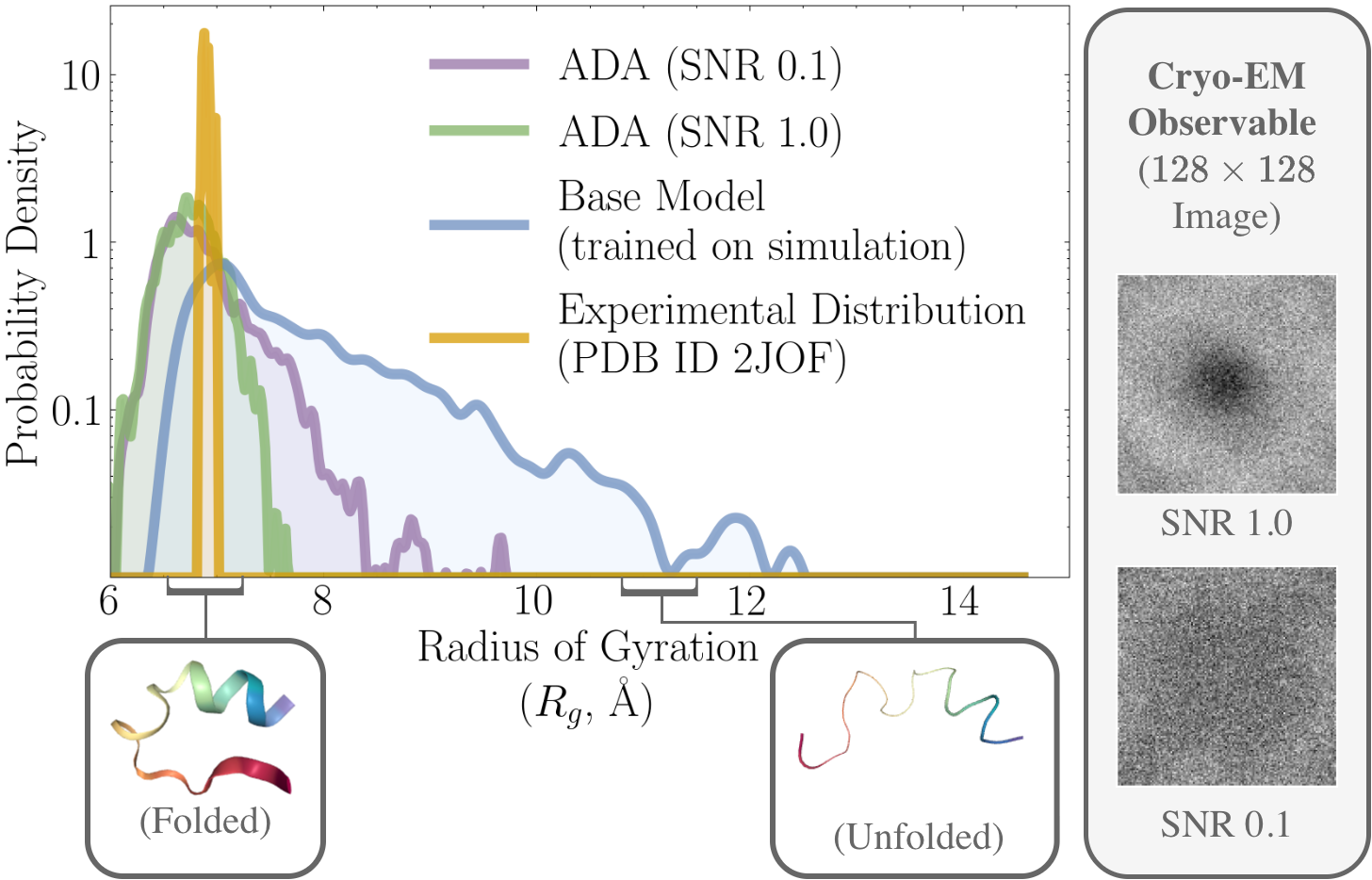}
    \caption{\textbf{\methodname shifts the simulated distribution to align with the experimentally measured Trp-cage structures by using only noisy, high-dimensional cryo-EM images.} The base generative model is trained on MD simulations using a classical force field, which samples unfolded states more often than the experimental ensemble (2JOF PDB). \methodname fine-tunes the base model using the cryo-EM images at SNR 1.0 and 0.1. We evaluate on a held-out observable, the radius of gyration, which is not used during alignment. At both noise levels, ADA removes most of the unfolded tail and shifts the density toward the experimental structures.}
    \label{fig:rg_shift_trp}
    \vspace{-12pt}
\end{wrapfigure}

Finally, we use the coarse-grained fast-folding proteins \citep{Majewski2023_torch_md_thermo, LindorffLarsen2011_fast_folders} as a test bed for aligning to experimental data using measurements from the Protein Data Bank (PDB) \citep{pdb}. We first train a base model to match the Boltzmann distribution induced by molecular dynamics simulations using a classical force field \citep{Majewski2023_torch_md_thermo, LindorffLarsen2011_fast_folders}. Starting from this base model, we then use experimental structures of the Trp-cage (ID 2JOF) and BBL (ID 2WXC) proteins from the PDB as our target distribution \citep{Wthrich1986_nmr}. 

For this experiment, \methodname uses cryo-EM images as observables. We simulate the image formation from the reference experimental structures \citep{Zhong2021_cryodrgn}, yielding noisy, high-dimensional observables ($128 \times 128$ pixels; \Cref{fig:rg_shift_trp} and \Cref{fig:cryo_images}). We choose this evaluation setting since it provides a challenging benchmark that measures \methodname's ability to improve a prior derived from classical simulations using a small set ($<30$ structures) of real-world structures. Simulating the forward cryo-EM imaging enables us to control the level of noise of the observable and allows us to provide precise evaluation metrics by comparing to the reference experimental structures.

For evaluation, we report the $W_1$ distance to other structural observables to see how \methodname improves alignment to held-out observables. To measure \methodname's ability to recover the underlying full state distribution from the cryo-EM images, we also measure the maximum Kabsch-aligned RMSD to the experimental amino acid positions. We provide further details in \Cref{apx:exp_details_protein}.

\paragraph{Results.} \Cref{fig:rg_shift_trp} and \Cref{tab:proten_alignment} show how \methodname aligns the base generative model to the distribution of experimental protein structures using the cryo-EM image observable. As in the molecular experiments, \methodname decreases the Wasserstein distance on held-out observables by up to $86\%$. Beyond observables, \methodname also reduces the maximum Kabsch-aligned RMSD to the experimental structures, indicating fewer extreme deviations in the underlying amino acid positions. Even when we decrease the signal-to-noise ratio (SNR) for the cryo-EM observable, \methodname improves the base distribution's alignment with the experimental data. 

This experiment shows the importance of \methodname's ability to perform full \textit{distributional} alignment. At low SNR, matching only the expectation of the cryo-EM observable would just match the average noise of the images, and does not improve alignment (see \Cref{tab:proten_alignment}). Similarly, imposing a sufficient number of moment constraints for the $128\times128$ dimensional distribution would be impractical, as shown by the synthetic experiment in \Cref{exp:synthetic}. Appropriately scaled, this experiment suggests the potential of \methodname to leverage a large dataset of noisy experimental measurements to bridge the simulation-to-experiment gap.  

\paragraph{Inferring unknown parameters of the observable forward model.} In a real-world imaging setup, the exact parameters of the cryo-EM imaging model might be unknown. We conduct an additional experiment using Trp-cage where \methodname must recover part of the observable (forward) model itself, instead of knowing the exact image formation process. We create each cryo-EM image at a signal-to-noise ratio (SNR) drawn from $\mathcal{N}(1,\,0.2)$. The model is provided only with an informed prior over the SNR, $\mathcal{N}(0.8,\,0.4)$, since the typical SNR of a cryo-EM measurement is often known but can vary slightly across measurements. 

In addition to the ADA objective, the model now variationally infers the SNR distribution, learning a variational posterior over the SNR (regularized toward the prior via a KL term) while aligning to the observable. We formalize this in the ADA framework by considering the underlying states to be tuples of the protein conformation and noise level of the measurement device. On held-out observables, ADA reduces the Wasserstein distance to the experimental radius-of-gyration distribution from 1.07 to 0.25, and from 1.26 to 0.26 for the mean interatomic distance distribution, while the inferred SNR concentrates near its true value (learned $\mathcal{N}(0.98, 0.19)$). We visualize this procedure in \Cref{fig:learning_snr}.

\begin{figure}
    \centering
    \includegraphics[width=\linewidth]{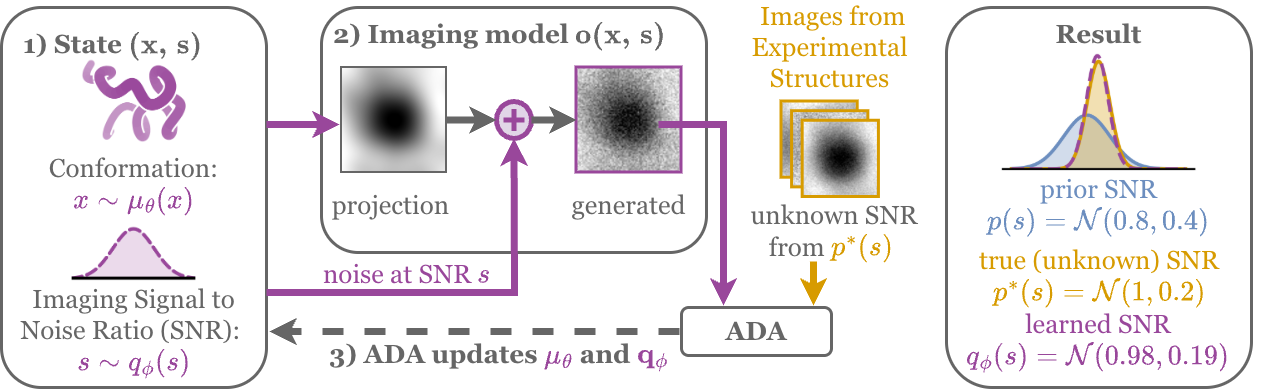}
    \caption{\textbf{ADA jointly infer parameters of the observable forward model while aligning to the experimental distribution.} In a more realistic imaging scenario, the parameters of the cryo-EM forward model may not be known exactly. We run an experiment where ADA must infer the SNR of the image formation process while also aligning the distribution of generated conformations with the experimental structures. ADA now receives only an informed prior over the possible SNRs, and we sample a new SNR for each generated image. ADA recovers the correct SNR distribution while still improving the alignment of the generated structures.}
    \label{fig:learning_snr}
\end{figure}

\vspace{-8pt}
\section{Conclusion}
\vspace{-4pt}
In this work, we introduced \methodname, an iterative algorithm for bridging the gap between simulation and experiment that alternates between distinguishing low- and high-fidelity samples and refining a generative model accordingly. We provided both theoretical guarantees and empirical evidence that \methodname can align a classical force field with simulated experimental measurements, and that its performance improves as additional observables are incorporated, even when these observables are correlated.

While our experiments focus on observables derived from a single state, the proposed objective naturally extends to dynamic observables, such as an autocorrelation function, presenting an important direction for future work. Moreover, although we evaluate \methodname on widely used benchmarks, these datasets remain a simplified proxy for the scale and complexity encountered at the frontier of biomolecular research. Real experimental data are also often noisy and expensive to obtain, posing additional challenges that must be addressed to enable broader deployment.

At the same time, the observation that \methodname improves with an increasing number of observables suggests that scaling \methodname with more computational resources and experimental datasets could further improve performance. We also hope that the development of algorithms capable of leveraging experimental measurements at scale will motivate the creation of more standardized and widely accessible experimental datasets. Finally, while our experiments are grounded in computational chemistry, we emphasize that \methodname is not domain-specific and provides a general framework for learning from partial observations. By combining structured simulated priors with real-world measurements, \methodname represents a step toward models that more faithfully align with the real world.

\paragraph{Acknowledgments.} We thank James Bowden and Sanjeev Raja for feedback on the manuscript. We thank Nithin Chalapathi for feedback on the cryo-EM experiments and Michael Psenka for feedback on the proofs. This research was initially supported by Laboratory Directed Research and Development (LDRD) funding under the Department of Energy Contract Number DE-AC02-05CH11231 and later supported by the U.S. Department of Energy, Office of Science, Office of Advanced Scientific Computing Research under contract number DE-SC0026116, DARPA TIAMAT, and the Toyota
Research Institute as part of the Synthesis Advanced Research Challenge.

\newpage
\bibliographystyle{plainnat}
\bibliography{references}

\newpage
\appendix
\section{Experimental Details}
\label{apx:exp_details}

\begin{table*}[ht]
\caption{\textbf{Observable Alignment naturally handles multiple observables.} Our method learns the marginal distribution of multiple, potentially correlated observables directly from data via a discriminator. In contrast, guidance-based approaches typically target a single observable and offer no guarantees for aligning multiple (potentially correlated) observables. To match the distribution of a single observable, guidance would need to rely on conditioning on samples from the target observable distribution, whereas Observable Alignment learns the distribution directly from data.}
\centering
\resizebox{\textwidth}{!}{%
\begin{tabular}{c cc cc c}
\toprule
\textbf{Method}
& \multicolumn{2}{c}{\textbf{Match Expectation}}
& \multicolumn{2}{c}{\textbf{Match Distribution}}
& \textbf{Requires Training} \\
& Single & Multiple
& Single & Multiple
& \\
\midrule
\textbf{Conditional Generative Model}
& $\checkmark$ & \xmark
& $\checkmark$ & \xmark
& classifier (or CFG) \\

\textbf{Expectation Alignment}
& \checkmark & \checkmark
& \xmark & \xmark
& generator \\

\textbf{\methodname} (ours)
& \checkmark & \checkmark
& \checkmark & \checkmark
& generator + discriminator \\
\bottomrule
\end{tabular}%
}
\label{tab:method_capabilities}
\vspace{-6pt}
\end{table*}

We provide further details on our experiments. Unless otherwise specified, we use a Transformer encoder \citep{vaswani2023attentionneed} for the generative model and an MLP for our discriminators. We apply rotation augmentation in all our experiments. We used Adam \citep{kingma2017adammethodstochasticoptimization} for all of our experiments and normalized both the molecular and observable inputs. We also found that using a gradient penalty to handle the Lipschitz constraint to be more flexible than fixing a spectral norm. We discuss individual experiments below.     

\subsection{Synthetic Mixture-of-Gaussians}

We visualize the synthetic mixture-of-Gaussians benchmark in \Cref{fig:mix_guassians}. Although this is a simple distribution in three dimensions, EA struggles to align to the target distribution (see \Cref{exp:synthetic} and \Cref{tab:alignment_metrics}).

\begin{figure}[ht]
    \centering
    \includegraphics[width=0.5\linewidth]{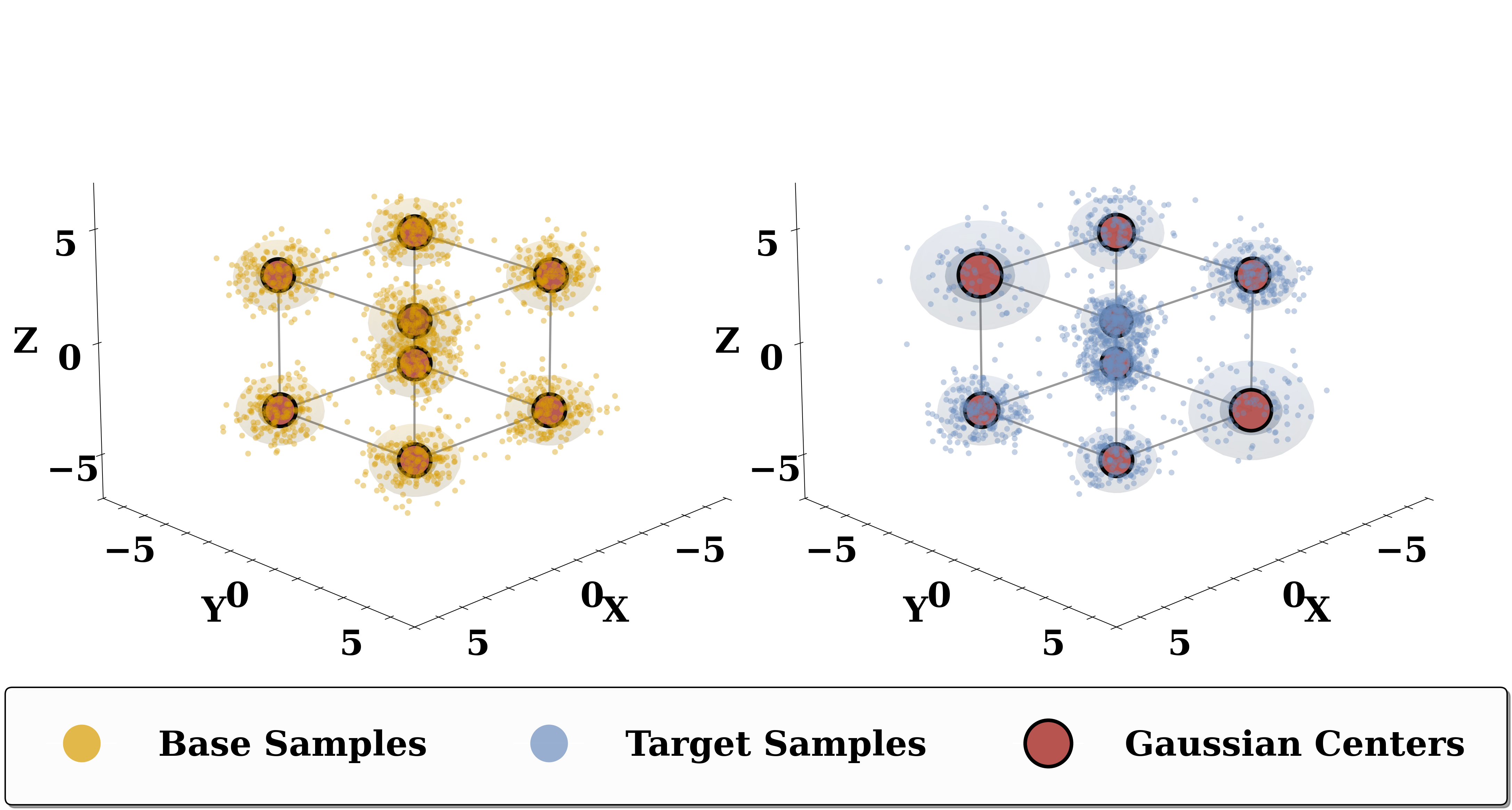}
    \caption{\textbf{Synthetic mixture-of-Gaussians benchmark.} The base distribution (left) is an isotropic mixture-of-Gaussians with eight components located at the vertices of a cube, with uniform mixture weights. We define the target distribution (right) by perturbing the variances and mixture weights for the Gaussians (for example, larger variance on top-left to bottom-right diagonal).}
    \label{fig:mix_guassians}
\end{figure}

\paragraph{Dataset Parameters}
The synthetic benchmark consists of two three-dimensional mixture-of-Gaussians distributions, both with eight components centered at the vertices of a cube with side length 6.0~\AA. We summarize the key parameters in \Cref{tab:synthetic_params}.

\begin{table}[ht]
\centering
\caption{\textbf{Synthetic mixture-of-Gaussians parameters.} Both distributions have eight Gaussian components positioned at the cube vertices $\{(\pm 3, \pm 3, \pm 3)\}$~\AA. The base distribution is uniform, while the target distribution has heterogeneous variances and mixture weights.}
\label{tab:synthetic_params}
\setlength{\tabcolsep}{8pt}
\begin{tabular}{lcc}
\toprule
\textbf{Parameter} & \textbf{Base (Uniform)} & \textbf{Target (Reweighted)} \\
\midrule
Cube Side Length & 6.0~\AA & 6.0~\AA \\
Number of Components & 8 & 8 \\
Component Positions & Cube vertices & Cube vertices \\
\midrule
\textbf{Mixture Weights:} & & \\
Diagonal Pair 1 & 0.25 & 0.40 \\
Diagonal Pair 2 & 0.25 & 0.30 \\
Diagonal Pair 3 & 0.25 & 0.10 \\
Diagonal Pair 4 & 0.25 & 0.20 \\
\midrule
\textbf{Variances} ($\sigma^2$, \AA$^2$): & & \\
Vertex 1: $(+3,+3,+3)$ & 0.5 & 0.3 \\
Vertex 2: $(+3,+3,-3)$ & 0.5 & 0.5 \\
Vertex 3: $(+3,-3,+3)$ & 0.5 & 1.0 \\
Vertex 4: $(+3,-3,-3)$ & 0.5 & 0.4 \\
Vertex 5: $(-3,+3,+3)$ & 0.5 & 0.6 \\
Vertex 6: $(-3,+3,-3)$ & 0.5 & 1.2 \\
Vertex 7: $(-3,-3,+3)$ & 0.5 & 0.5 \\
Vertex 8: $(-3,-3,-3)$ & 0.5 & 0.3 \\
\bottomrule
\end{tabular}
\end{table}

\begin{table}[ht]
\centering
\caption{\textbf{Distributional alignment metrics on mixture of Gaussians.} We compare alignment using \methodname against EA style approaches using an observable set of 2-dimensional projection maps. Up until the fourth order, moment-based matching is unable to recover the distribution empirically.}
\label{tab:alignment_metrics}
\setlength{\tabcolsep}{4pt}
\begin{tabular}{lcc}
\toprule
Method & Energy W2 $\downarrow$ & Cluster KL $\downarrow$ \\
\midrule
Base
& 1.10 & 0.12 \\

\addlinespace[2pt]
\midrule
\textbf{\methodname (ours)}
& \textbf{0.05} & \textbf{0.01} \\
\midrule

EA (1st) & 1.10 & 0.12 \\
EA (2nd) & 0.90 & 0.02 \\
EA (3rd) & 0.90 & 0.02 \\
EA (4th) & 0.70 & 0.05 \\
\bottomrule

\end{tabular}
\end{table}

\subsection{Small Molecules}
\label{apx:observables}

\begin{figure}
    \centering
    \includegraphics[width=0.35\linewidth]{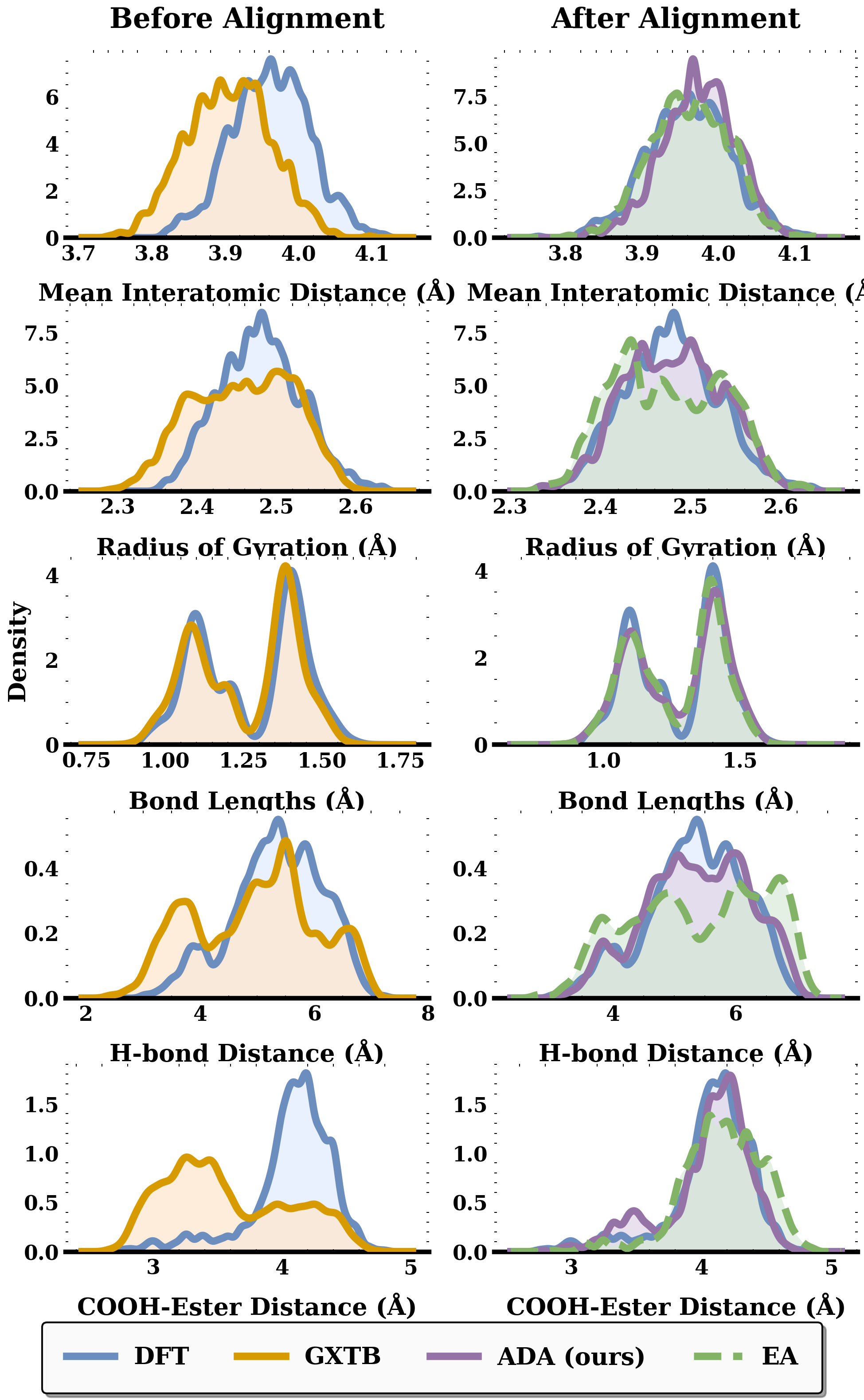}
    \caption{\textbf{Comparison of expectation alignment (EA) and \methodname across all evaluated observables on the MD17 aspirin system.} While EA aligns first-order statistics, \methodname more accurately captures higher-order structure, yielding closer agreement with the full target observable distributions.}
    \label{fig:experiment2_hist}
\end{figure}

\begin{table*}[t]
\centering
\caption{\textbf{\methodname improves in alignment to the target distribution when given access to an increasing number of observables.} We train a base model on constant-temperature molecular dynamics simulations using the semi-empirical GFN2-xTB potential. This model is then aligned with the distribution induced by a more accurate DFT potential using structural observables. Each column reports the distance to the reference DFT observable distribution. Moving down the rows of the table adds observables to the alignment procedure, with the bottom row representing aligning with the full observable set: mean interatomic distance (MD), radius of gyration ($R_g$), bond lengths, hydrogen-bond distance, and COOH--ester distance. The upper-triangular part of the left side of the table represents improvements on held-out observables (highlighted in purple). We additionally report improvement on Jensen-Shannon divergences on three fully held-out free energy surfaces: (MD, Rg), (bond length, H-bond distance), and (H-bond distance, COOH–ester distance). These surfaces are not used during training (indicated by the purple highlight) and therefore measure preservation of joint molecular structure beyond the enforced observable marginals. Compared to Expectation Alignment (EA), \methodname better aligns to the target distribution (indicated with bold), especially with additional observables. %
}

\scriptsize
\setlength{\tabcolsep}{2.5pt}
\renewcommand{\arraystretch}{1.05}

\begin{tabular}{l cc|cc|cc|cc|cc || cc|cc|cc}
\toprule
& \multicolumn{10}{c}{Observable marginals $W_1 \downarrow$}
& \multicolumn{6}{c}{Held-out FES JSD $\downarrow$} \\
\cmidrule(lr){2-11}\cmidrule(lr){12-17}

Constraint
& \multicolumn{2}{c}{MD}
& \multicolumn{2}{c}{$R_g$}
& \multicolumn{2}{c}{Bond}
& \multicolumn{2}{c}{Hbond}
& \multicolumn{2}{c}{COOH}
& \multicolumn{2}{c}{A}
& \multicolumn{2}{c}{B}
& \multicolumn{2}{c}{C} \\

& Ours & EA & Ours & EA & Ours & EA & Ours & EA & Ours & EA
& Ours & EA & Ours & EA & Ours & EA \\

\midrule

Base
& \multicolumn{2}{c}{0.06}
& \multicolumn{2}{c}{0.02}
& \multicolumn{2}{c}{0.018}
& \multicolumn{2}{c}{0.40}
& \multicolumn{2}{c}{0.50}
& \multicolumn{2}{c}{0.22}
& \multicolumn{2}{c}{0.44}
& \multicolumn{2}{c}{0.38} \\

\midrule

+ MD
& \textbf{0.008} & 0.012
& \cellcolor{lightpurple}0.020 & \cellcolor{lightpurple}\textbf{0.009}
& \cellcolor{lightpurple}\textbf{0.007} & \cellcolor{lightpurple}0.010
& \cellcolor{lightpurple}\textbf{0.25} & \cellcolor{lightpurple}0.41
& \cellcolor{lightpurple}\textbf{0.37} & \cellcolor{lightpurple}0.39
& \cellcolor{lightpurple}0.23 & \cellcolor{lightpurple}\textbf{0.20}
& \cellcolor{lightpurple}\textbf{0.40} & \cellcolor{lightpurple}0.41
& \cellcolor{lightpurple}\textbf{0.35} & \cellcolor{lightpurple}\textbf{0.35} \\

+ RoG
& \textbf{0.009} & 0.012
& 0.010 & \textbf{0.009}
& \cellcolor{lightpurple}\textbf{0.010} & \cellcolor{lightpurple}0.011
& \cellcolor{lightpurple}\textbf{0.30} & \cellcolor{lightpurple}0.41
& \cellcolor{lightpurple}\textbf{0.43} & \cellcolor{lightpurple}0.48
& \cellcolor{lightpurple}0.20 & \cellcolor{lightpurple}\textbf{0.19}
& \cellcolor{lightpurple}0.42 & \cellcolor{lightpurple}\textbf{0.40}
& \cellcolor{lightpurple}0.35 & \cellcolor{lightpurple}\textbf{0.34} \\

+ Bond
& \textbf{0.004} & 0.019
& \textbf{0.005} & 0.009
& \textbf{0.004} & 0.009
& \cellcolor{lightpurple}\textbf{0.21} & \cellcolor{lightpurple}0.36
& \cellcolor{lightpurple}\textbf{0.46} & \cellcolor{lightpurple}0.48
& \cellcolor{lightpurple}\textbf{0.15} & \cellcolor{lightpurple}0.20
& \cellcolor{lightpurple}\textbf{0.41} & \cellcolor{lightpurple}0.43
& \cellcolor{lightpurple}\textbf{0.36} & \cellcolor{lightpurple}\textbf{0.36} \\

+ Hbond
& \textbf{0.007} & 0.017
& \textbf{0.009} & 0.010
& \textbf{0.006} & 0.011
& \textbf{0.06} & 0.19
& \cellcolor{lightpurple}\textbf{0.39} & \cellcolor{lightpurple}\textbf{0.39}
& \cellcolor{lightpurple}\textbf{0.13} & \cellcolor{lightpurple}0.36
& \cellcolor{lightpurple}\textbf{0.39} & \cellcolor{lightpurple}0.43
& \cellcolor{lightpurple}\textbf{0.33} & \cellcolor{lightpurple}0.36 \\

+ COOH
& \textbf{0.009} & \textbf{0.009}
& \textbf{0.002} & 0.012
& \textbf{0.006} & 0.011
& \textbf{0.06} & 0.31
& \textbf{0.04} & 0.06
& \cellcolor{lightpurple}\textbf{0.08} & \cellcolor{lightpurple}0.17
& \cellcolor{lightpurple}\textbf{0.13} & \cellcolor{lightpurple}0.20
& \cellcolor{lightpurple}\textbf{0.09} & \cellcolor{lightpurple}0.19 \\

\bottomrule
\end{tabular}
\label{tab:cumulative_ablation}

\end{table*}

In this work, we consider the aspirin system (C$_9$H$_8$O$_4$) composed of 21 atoms, with DFT data taken from the MD17 dataset \cite{md17}. To characterize the conformational landscape and hydrogen bonding dynamics of aspirin, we compute the following set of observables:

\paragraph{Radius of Gyration} The radius of gyration $R_g$ quantifies the spatial extent of the molecular structure and is defined as:
\begin{equation}
R_g^2 = \frac{1}{M} \sum_{i=1}^{N} m_i \|\mathbf{r}_i - \mathbf{r}_{\text{COM}}\|^2
\end{equation}
where $m_i$ and $\mathbf{r}_i$ are the mass and position of atom $i$, $N=21$ is the total number of atoms, $M = \sum_{i=1}^{N} m_i$ is the total mass, and $\mathbf{r}_{\text{COM}}$ is the center of mass. This observable captures the overall compactness of the molecule, with lower values indicating more compact conformations and higher values indicating more extended structures. Units: $\AA$.

\paragraph{Mean Interatomic Distance} The mean interatomic distance $\langle d \rangle$ provides a global measure of the molecular size:
\begin{equation}
\langle d \rangle = \frac{2}{N(N-1)} \sum_{i=1}^{N-1} \sum_{j=i+1}^{N} \|\mathbf{r}_i - \mathbf{r}_j\|
\end{equation}
This quantity averages over all pairwise atomic distances and serves as a complementary measure to $R_g$, being sensitive to the overall distribution of atoms rather than their distance from the center of mass. Units: \AA.

\paragraph{Bond Lengths} We monitor all covalent bond lengths $d_{ij}$ between bonded atoms $i$ and $j$:
\begin{equation}
d_{ij} = \|\mathbf{r}_i - \mathbf{r}_j\|
\end{equation}
These observables track local geometric distortions and vibrational modes within the molecular framework. For aspirin, we focus particularly on the bonds within the carboxylic acid (COOH) and ester functional groups. Units: \AA.

\paragraph{Hydrogen Bond Distance} The intramolecular hydrogen bond distance $d_{\text{HB}}$ is defined as the distance between the hydroxyl hydrogen of the carboxylic acid group and the carbonyl oxygen of the ester group:
\begin{equation}
d_{\text{HB}} = \|\mathbf{r}_{\text{H(COOH)}} - \mathbf{r}_{\text{O(ester)}}\|
\end{equation}
This observable is crucial for characterizing the intramolecular hydrogen bonding that stabilizes certain aspirin conformers. Shorter distances indicate stronger hydrogen bonding interactions. Units: \AA.

\paragraph{COOH-Ester Distance} The COOH-ester distance $d_{\text{CE}}$ measures the separation between the center of mass of the carboxylic acid group and the center of mass of the ester group:
\begin{equation}
d_{\text{CE}} = \|\mathbf{r}_{\text{COM}}^{\text{COOH}} - \mathbf{r}_{\text{COM}}^{\text{ester}}\|
\end{equation}
This collective variable captures the relative orientation and proximity of the two functional groups, which is directly related to the molecule's ability to form the intramolecular hydrogen bond. Units: \AA.

\subsection{Free Energy Surfaces}

To explore the conformational landscape of aspirin, we compute two-dimensional free energy surfaces (FES) as functions of pairs of these observables:

\begin{itemize}
\item \textbf{FES A}: $F(d_{\text{HB}}, R_g^2)$ -- This surface relates hydrogen bonding strength to molecular compactness, revealing how intramolecular hydrogen bonds affect the overall molecular geometry.

\item \textbf{FES B}: $F(d_{\text{CE}}, d_{\text{HB}})$ -- This surface directly connects the functional group separation to the hydrogen bond distance, capturing the geometric requirements for hydrogen bond formation.

\item \textbf{FES C}: $F(\xi_1, \xi_2)$ -- This surface explores the functional group geometry using appropriate collective variables that characterize the relative orientation and configuration of the COOH and ester moieties.
\end{itemize}

We take the Jensen-Shannon divergence of the free energy surfaces by a standard binning procedure of the 2-dimensional distributions.

\subsection{Proteins}
\label{apx:exp_details_protein}

\begin{table}[t]
\centering
\caption{
\textbf{\methodname aligns a generative model trained on simulated data from a classical force field with experimental protein structures from the PDB using, noisy, high-dimensional cryo-EM images.} We pre-train a generative model on samples of fast-folding proteins obtained from molecular dynamics simulations using a classical force field \citep{Majewski2023_torch_md_thermo}. Using only the noisy cryo-EM images, \methodname improves alignment with the underlying experimental distribution of structures, as measured by the Wasserstein distance to other, held-out observables. The maximum Kabsch-aligned RMSD to the experimental structures also decreases, indicating better distributional alignment at level of amino acid positions.
}
\small
\setlength{\tabcolsep}{3pt}
\renewcommand{\arraystretch}{1.15}

\begin{tabular}{lccl}
\toprule
{} &
\multicolumn{2}{c}{\textbf{$W_1$ to Experiment ($\downarrow$)}} & \textbf{Max RMSD}\\
&
\textbf{MeanDist}&
$\mathbf{R_g}$ & \textbf{(nm, $\downarrow$)}\\ \midrule

\textbf{Trp-Cage Base} &
1.26&
1.07 & 1.18\\
 EA (SNR 1.0, 1st order moments)& 1.28& 1.02&1.17\\
 EA (SNR 1.0, 1st$+$2nd order moments)& 1.25& 1.08&1.17\\

\methodname (SNR 1.0)&
0.17&
0.21 & 0.91\\
 \methodname (SNR 0.1)& 0.28&0.36 & 0.96\\ \hline

\textbf{BBL Base} &
1.36&
1.08 & 1.76\\
 EA (SNR 1.0, 1st order moments)& 1.23& 0.98&1.73\\
 EA (SNR 1.0, 1st$+$2nd order moments)& 1.26& 1.01&1.71\\
 \methodname (SNR 1.0)& 0.21&0.18 & 1.38\\

\methodname (SNR 0.1)&
0.27&
0.33 & 1.36\\\bottomrule
\end{tabular}

\vspace{1mm}
\label{tab:proten_alignment}
\end{table}

In this work, we considered two proteins: Trp-Cage and BBL with 20 and 47 amino acids, respectively. We used the same experimental structures, derived by NMR, used by \citet{Majewski2023_torch_md_thermo} with the following PDB IDs: 2JOF for Trp-cage (29 structures) and 2WXC for BBL (21 structures). We first pre-trained the base model using 100,000 samples of each of the proteins obtained by running MD simulations with a classical force field \citep{Majewski2023_torch_md_thermo}. We then used \methodname individually on each protein to match the experimental distribution of structures by using cryo-EM images as observables. We note that the simulation data was coarse-grained using the alpha-carbon \citep{Majewski2023_torch_md_thermo}, so we applied the same alpha-carbon coarse-graining to the experimental data.

\begin{wrapfigure}{R}{0.4\textwidth}
    \centering
    \includegraphics[width=0.6\linewidth]{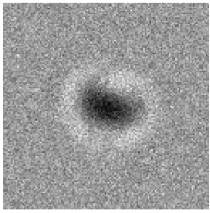}
    \caption{\textbf{Since we worked with small proteins, we ran an additional experiment with a defocus between $0.2$ and $0.5$ $\mu$m and a pixel size of $0.5$Å to emulate different imaging scenarios.}}
    \label{fig:cryo_pt5}
\end{wrapfigure}

We simulate the cryo-EM images to provide precise control over the image formation, allowing us to evaluate \methodname's robustness to noise. We build off the code provided by \citet{Zhong2021_cryodrgn} to generate $128 \times 128$ cryo-EM images. We use a standard deviation of $2.5$ for the Gaussian kernels for volume generation, an amplitude of $0.1$, a defocus between $0.8$ and $2.5$ $\mu$m , a b-factor between $1$ and $100$Å, and a pixel size of $0.3$Å. We also run experiments with a signal-to-noise ratio (SNR) of $1.0$ and $0.1$. We generate a new cryo-EM image each batch, mimicking a real cryo-EM imaging setup where there might be multiple different images of similar underlying structures. See \Cref{fig:cryo_images} for example images. 

We also run an additional experiment with a defocus between $0.2$ and $0.5$ $\mu$m and a pixel size of $0.5$Å since experimental images for small proteins can require different simulation parameters than larger ones (see \Cref{fig:cryo_pt5}). We see very similar results to the parameters described above, with a decrease in Wasserstein distance of more than $80\%$ to held out observables compared to the base model.

Note that for this experiment, we use a convolutional neural network (CNN) for the discriminator. The convolutional encoder consists of 3 convolutional blocks, each containing a convolution (kernel size 3, padding 1) and max pooling (stride 2). The channels are $1,32,64,128$. An adaptive average pooling layer reduces spatial dimensions to a fixed size of $2\times2$, followed by a linear projection. We report further hyperparameters in \Cref{tab:hyperparameters}.

\begin{figure}
    \centering
    \includegraphics[width=0.9\linewidth]{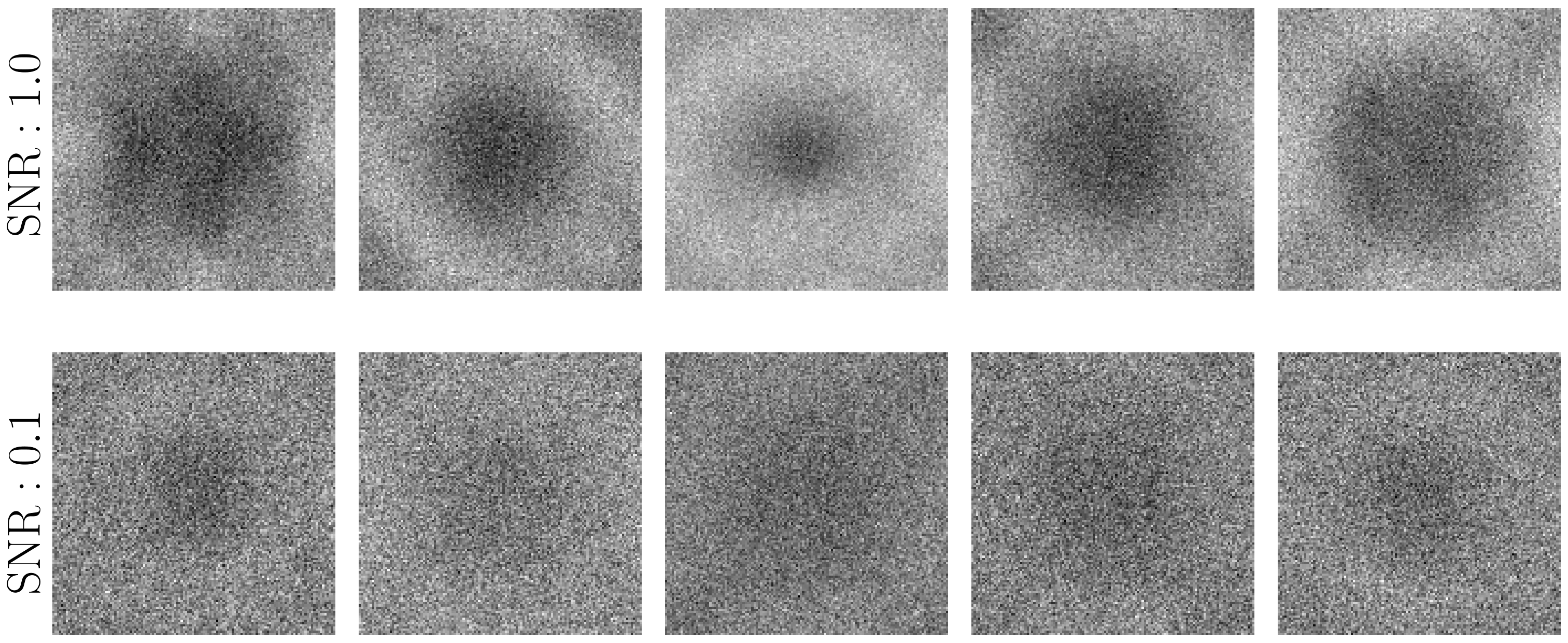}
    \caption{\textbf{Example cryo-EM images of Trp-cage used by \methodname to align with the experimental distribution of protein structures.}}
    \label{fig:cryo_images}
\end{figure}

\begin{table*}[t]

\centering
\caption{\textbf{Experiment Hyperparameters.} Note that discriminator steps refers to the number of  gradient steps taken between each update of the generative model, and the discriminator weight refers to the weighting of the discriminator term relative to the KL term in the generative model update step.} 
\resizebox{\textwidth}{!}{%
\begin{tabular}{llll}
\toprule
\textbf{}                        & \multicolumn{3}{c}{\textbf{Value}}                                                    \\
\textbf{Hyperparameter}          & \multicolumn{1}{c}{\textbf{Synthetic}} & \textbf{Small Molecules} & \textbf{Proteins} \\ \hline
\textbf{Base Model:}             &                                        &                          &                   \\
Layers                           &                    MLP(512, 512, 512, 512)                    &              8            & 8                 \\
Heads                            &                     N/A                   &             12             & 12                \\
Hidden Dimension                 &                      N/A                  &           384               & 384               \\
Number of Euler Steps (sampling) &                       N/A                 &              45            & 35                \\ \hline
\textbf{Discriminator Model:}             &                                        &                          &                   \\
Hidden Dimension                 &                      512                  &           512               & (CNN, see \ref{apx:exp_details_protein})\\
Layers                           &                      2                  &             2             & (CNN, see \ref{apx:exp_details_protein})\\ \hline
\textbf{\methodname:}                 &                                        &                          &                   \\
Discriminator Weight&     128.0                                   &      4096.0                         &                   512\\
Discriminator Steps              &       1                                 &           1               &                   3\\
Gradient Penalty Weight          &                     1000                     &              1000              & $1\times10^6$\\ \hline
\textbf{Training:}               &                                        &                          &                   \\
Generative Model LR (Pre-training, fine-tuning)              &                        $1\times10^{-3}, 1\times10^{-5}$  &         $1\times10^{-3}, 1\times10^{-5}$                 & $1\times10^{-3}, 1\times10^{-5}$\\
Discriminator LR                 &                       $1\times10^{-3}$                  &               $1\times10^{-3}$            & $1\times10^{-3}$  \\
Weight Decay (Pre-training, fine-tuning)                      &                 $1\times10^{-4}, 1\times10^{-8}$                        &       $1\times10^{-4}, 1\times10^{-8}$                    & $1\times10^{-4}, 1\times10^{-8}$  \\
\# Gradient Steps (Pre-training, fine-tuning)                &          10,000, 10,000        &     500,000, 10,000  & 300,000, 1,000\\
Batch Size                       &    1024                                    &      64                    & 128               \\ \bottomrule
\end{tabular}
}
\label{tab:hyperparameters}
\end{table*}

\subsection{Computational Resources}

All experiments were run on a single GPU. We used both A6000s and H100s to run our experiments. Training the base models and applying \methodname took less than 24 GPU hours across all of our experiments. Interestingly, performance kept improving with \methodname as the models kept training longer, especially with more observables and more complicated systems. This suggests that \methodname could benefit and take advantage of further computation on larger experimental datasets. 

\section{Proofs}
\label{apx:proofs}
We provide the full proofs of the theorems presented in the main text.

\subsection{Proof of \cref{thm:saddle-point}}
\label{subsec:saddle_point_prof}
\begin{theorem}[Existence and uniqueness of the saddle point]
\label{thm:saddle-point}
Assume $X$ is a compact Polish space and that $\mb$ has full support on $X$. Additionally, allow $o^{(i)}$ to be continuous for $i \in I$. For any $\beta \in \mathbb{R}$, \cref{eqn:minimax} admits a saddle point $(\mu^*, (f^{(i)})^*)$, where $\mu^*$ is unique. Consequently, the recovered distribution is well defined. We denote this solution by $\mu^*\in \solL{}$.
\end{theorem}

\begin{proof}

\textit{The proof is straightforward Sion's theorem. We start by showing that the space of probability measures is compact with respect to the weak* topology. We then note that Lipschitz functions are a convex subset of the continuous functions, equipped with the supremum norm. With this, both our probability measures (generators) and functions (discriminators) live on convex subsets of linear topological vectors spaces. We then use basic facts about the continuity of our Lagrangian to justify the existence of a unique saddle point $\mu^*$.} 

\textit{We first equip the space of probability measures with the weak* topology; allowing $\mu \mapsto\braket{\mu, f} = \int f \; d\mu: \mathcal P \rightarrow \mathbb R$ to be continuous; resulting in $\mathcal P_1$ being a compact convex subset of a linear topological vector space with respect to the weak* topology.}

Let $\mathcal{P} = \mathcal{M}(X)$ denote the space of Radon measures on $(X,\mathcal{F})$, and let $\mathcal{P}_1 \subset \mathcal{P}$ be the subset of probability measures. 

Allow $C_0(X)$ to be the set of continuous functions on $X$ vanishing at infinity. As $X$ is compact we have that $C_0(X) = C(X)$; the set of all continuous functions.

We equip $\mathcal P$ with the weak* topology induced by $C(X)^*$, i.e., the coarsest topology such that for all $f \in C(X)$,
\[
\mu \mapsto \langle \mu, f \rangle := \int f \, d\mu,
\]
is continuous. We have that $\mathcal{P}$ is a linear topological vector space, and $\mathcal{P}_1$ is a convex subset of $\mathcal P$.

We note that probability measures are known to be compact with respect to the weak*. $C_0(X)$ is a Banach space, thus the closed TV unit ball $B_1  = \{ \mu : \|\mu\|_{TV} \leq 1\}$ is compact in weak*. As probability measures are closed in weak*, and $\mathcal P_1 \subset B_1$ we get compactness of $\mathcal P_1$ in weak*.

\textit{We further define the set of our rewards $\mathcal F^m$ such that it is a convex subset of a linear topological vector space.} 

Allow $f^{(i)} \in \mathrm{LIP}_1(O^{i}): O^{(i)} \rightarrow \mathbb{R}$ to be the discriminator associated with observable $i$. By considering the set of continuous functions $C(O^{(i)})$ under the supremum norm, we find that $\mathcal F_i \subset C(O^{(i)})$ is a convex subset of a linear topological vector space.

In our setup, we consider a vector of finite discriminators, one per observable. We denote this by: $\mathcal F^m = \prod_{i \in I} \mathcal (F_i)$, which is additionally a convex subspace of $\prod_{i \in I}  C(O^{(i)}) $.

Thus $\mathcal F^m \subset \prod_{i \in I}  C(O^{(i)}) $ is a convex subset linear topological vector space. 

\textit{We now finish, by showing that we can invoke Sion's theorem using the compactness of $\mathcal P_1$ , yielding the existence of a saddle point which is unique in $\mu^*$.}

The Lagrangian is given by

\[
\mathcal{L}(\mu,(f^{(i)}))
= \beta
\sum_{i\in I} \mathbb{E}_{o^{(i)}_\#((\mu-\nu))}[f^{(i)}]
-
D_{\mathrm{KL}}(\mu\|\mu_{base}).
\]

\textit{First, we demonstrate that $\mu \mapsto \mathcal{L}(\mu,(f^{(i)}))$ is upper semi-continuous in $\mu$, and strictly concave.}

The term $-D_{\mathrm{KL}}(\mu\|\mb)$ is upper semi-continuous in the weak* topology and strictly concave in $\mu$.

Likewise, for fixed $(f^{(i)})$, the mapping $\mu \mapsto \sum_{i\in I} \mathbb{E}_{o^{(i)}_\#((\mu-\nu))}[f^{(i)}]$ is continuous in the weak* topology, as we have a finite sum of continuous functions.

\textit{Next, we demonstrate that $(f^{(i)}) \mapsto \mathcal{L}(\mu,(f^{(i)}))$ is lower semi-continuous in $(f^{(i)})$, and convex.}

For fixed $(\mu-\nu)$:
\[
(f^{(i)}) \mapsto \sum_{i\in I} \mathbb{E}_{o^{(i)}_\#(\mu-\nu)}[f^{(i)}]
\]
is continuous on $\mathcal{F}^m$ with respect to the product topology. 

This is because $\pi_i : (f(i)) \mapsto f^{(i)}$ is continuous and integration with respect to a finite measure $f(i) \mapsto \braket{\gamma, f(i)}$ is continuous, making  $\sum_{i\in I} \mathbb{E}_{o^{(i)}_\#((\mu-\nu))}[f^{(i)}] = [\sum_{i\in I} \braket{o^{(i)}_\#((\mu-\nu)), \pi_i(\cdot)}] (f^{(i)})$ continuous in $(f^{(i)})$.

Hence, we achieve lower semi-continuity in $(f^{(i)})$ trivially.

Convexity comes naturally from linearity.

\textit{We now invoke Sion's theorem, as we have proved the constituent lemmas.}

Therefore, $\mathcal{L}$ is concave and upper semi-continuous in $\mu$, and convex and lower semi-continuous in $(f^{(i)})$. 

Sion’s theorem applies, yielding
\[
\max_{\mu} \inf_{(f^{(i)})} \mathcal{L}(\mu,(f^{(i)}))
=
\inf_{(f^{(i)})} \max_{\mu} \mathcal{L}(\mu,(f^{(i)})).
\]
As $\mathcal P_1$ is compact, we get attainment of a maximizer $\mu^*$. Uniqueness of $\mu^*$ follows from the strict concavity of $-D_{\mathrm{KL}}(\mu\|\mb)$ in $\mu$.

\end{proof}

\subsection{Asymptotic Convergence in Wasserstein}
\label{subsec:asymptotic_converence}

\begin{theorem}[Asymptotic convergence in Wasserstein]
\label{thm:wasserstein-convergence}
Assume $D_{\mathrm{KL}}(\nu \| \mu_{base}) < \infty$. Let $(\mu^*_\beta, (f^{(i)}_\beta)^*) \in \solL{}$ for $\beta \in \mathbb{R}^+$. Then, as $\beta \to \infty$, we have
\[
d^{(i)}(\mu^*_\beta, \nu) \to 0
\quad \text{for all } i \in I.
\]
\end{theorem}

\begin{proof}

As $D_{\mathrm{KL}}(\nu \| \mu_{base}) < \infty$ there exists a constant $B \in \mathbb{R}$ such that
\[
D_{\mathrm{KL}}(\nu \| \mu_{base}) \le B.
\]

Since $\nu$ is feasible for the maximization problem, optimality of $\mu^*_\beta$ implies
\begin{equation}
\begin{aligned}
\mathcal{L}(\mu^*_\beta,(f^{(i)}_\beta)^*;\beta)
&\ge
\mathcal{L}(\nu,(f^{(i)}_\beta)^*;\beta) \\
&=
- D_{\mathrm{KL}}(\nu\|\mu_{base})
\ge -B .
\end{aligned}
\end{equation}

On the other hand, by definition of the saddle point,
\begin{equation}
\begin{aligned}
\mathcal{L}(\mu^*_\beta,(f^{(i)}_\beta)^*;\beta)
&=
- D_{\mathrm{KL}}(\mu^*_\beta\|\mu_{base})
- \beta \sum_{i\in I} d^{(i)}(\mu^*_\beta,\nu) \\
&\le
-\beta \sum_{i\in I} d^{(i)}(\mu^*_\beta,\nu),
\end{aligned}
\end{equation}
since the KL term is nonnegative.

Combining the two inequalities yields
\[
\beta \sum_{i\in I} d^{(i)}(\mu^*_\beta,\nu) \le B.
\]

In particular, for each $i\in I$,
\[
d^{(i)}(\mu^*_\beta,\nu) \le \frac{B}{\beta},
\]
which implies
\[
d^{(i)}(\mu^*_\beta,\nu) \to 0
\quad \text{as } \beta \to \infty.
\]

\end{proof}

\subsection{Weak Limit Converges in Constraint Set}
\label{subsec:weak_limit}

\begin{theorem}[Asymptotic convergence to the constraint set]
\label{thm:constraint-convergence}
Let $(\mu_\beta^*)_{\beta>0}$ be the sequence of saddle-point solutions from \cref{thm:wasserstein-convergence}. Suppose that $d^{(i)}(\mu_\beta^*,\nu)\to 0$ for all $i\in I$. Then any weak limit point $\bar{\mu}$ of $(\mu_\beta^*)$ satisfies
\[
(o^{(i)})_\# \bar{\mu} = (o^{(i)})_\# \nu
\quad \text{for all } i\in I,
\]
that is, $\bar{\mu}\in \mathcal{M}_o$.
\end{theorem}

\begin{proof}
Since $d^{(i)}(\mu_\beta^*,\nu)\to 0$, we have
\[
W_1\big((o^{(i)})_\#\mu_\beta^*,(o^{(i)})_\#\nu\big)\to 0
\]
for each $i$. Because $W_1$ metrizes weak convergence on probability measures over Polish spaces with finite first moment, it follows that
\[
(o^{(i)})_\#\mu_\beta^* \Rightarrow (o^{(i)})_\#\nu.
\]
By continuity of pushforward under weak convergence, any weak limit point $\bar{\mu}$ therefore satisfies
$(o^{(i)})_\#\bar{\mu}=(o^{(i)})_\#\nu$ for all $i$, completing the proof.
\end{proof}

\end{document}